\documentclass[11pt]{article}
\usepackage{graphicx} %

\usepackage{amsmath}
\usepackage[colorlinks=true, allcolors=blue]{hyperref}
\usepackage[linesnumbered,boxed,ruled,vlined]{algorithm2e}
\usepackage{amsfonts}
\usepackage{graphicx}
\usepackage{amsthm}
\usepackage{dsfont}
\usepackage{mathtools}
\usepackage{bbm}
\usepackage{cleveref}
\usepackage{complexity}
\usepackage{enumitem}
\usepackage{tikz}
\usetikzlibrary{positioning}
\usepackage{thm-restate}
\usepackage{float}
\usetikzlibrary{calc}
\usetikzlibrary{shapes,decorations.markings}
\usepackage{caption}
\usepackage{subcaption}

\usepackage[english]{babel}
\usepackage[T1]{fontenc}

\usepackage[margin=1.in]{geometry}

\usetikzlibrary{arrows,intersections}

\SetCommentSty{mycommfont}

\usepackage[colorinlistoftodos]{todonotes}

\newtheorem{theorem}{Theorem}
\newtheorem{claim}{Claim}

\newtheorem{lemma}{Lemma}
\newtheorem{corollary}{Corollary}

\newtheorem{remark}{Remark}

\theoremstyle{definition}
\newtheorem{definition}{Definition}

\newenvironment{proofof}[1]{\begin{proof}[{\textit{Proof of #1}}]}{\end{proof}}

\newcommand{\calA}{\mathcal{A}}
\newcommand{\calB}{{\mathcal{B}}}

\newcommand{\calD}{\mathcal{D}}
\newcommand{\calE}{{\mathcal{E}}}

\newcommand{\calH}{{\mathcal{H}}}
\newcommand{\calW}{{\mathcal{W}}}
\newcommand{\calQ}{{\mathcal{Q}}}
\newcommand{\calX}{{\mathcal{X}}}

\newcommand{\eps}{\varepsilon}

\newcommand{\Ex}{\mathop{\mathbf{E}}}
\renewcommand{\Pr}{\mathop{\mathbf{Pr}}}

\newcommand{\Cbias}{{C_{\mathrm{bias}}}}
\newcommand{\athre}{{\alpha_{\mathrm{thr}}}}
\newcommand{\clb}{{c_{\mathrm{lb}}}}

\makeatletter
\newcommand*\rel@kern[1]{\kern#1\dimexpr\macc@kerna}
\newcommand*\widebar[1]{%
  \begingroup
  \def\mathaccent##1##2{%
    \rel@kern{0.8}%
    \overline{\rel@kern{-0.8}\macc@nucleus\rel@kern{0.2}}%
    \rel@kern{-0.2}%
  }%
  \macc@depth\@ne
  \let\math@bgroup\@empty \let\math@egroup\macc@set@skewchar
  \mathsurround\z@ \frozen@everymath{\mathgroup\macc@group\relax}%
  \macc@set@skewchar\relax
  \let\mathaccentV\macc@nested@a
  \macc@nested@a\relax111{#1}%
  \endgroup
}
\makeatother

\renewcommand{\hat}{\widehat}
\renewcommand{\tilde}{\widetilde}

\newcommand{\VCDim}{\mathrm{VCDim}}
\newcommand{\Loss}{\mathcal{L}}

\usepackage{csquotes}

\newcommand{\defeq}{\mathrel{\mathop:}=}
\newcommand{\loss}{\mathcal{L}}
\newcommand{\strl}{\mathcal{A}}
\newcommand{\weakl}{\mathcal{W}}
\newcommand{\hatd}{{\hat{d}}}

\newcommand{\Ber}{\mathrm{Ber}}
\renewcommand{\R}{\mathbb{R}}

\newcommand{\agen}{\alpha_{\texttt{gen}}}

\newenvironment{proofsk}{%
  \proof}{\endproof}

\title{The Cost of Parallelizing Boosting}
\author{Xin Lyu \thanks{Department of EECS, UC Berkeley. Email: xinlyu@berkeley.edu. Supported by Avishay Tal's Sloan Research Fellowship, NSF CAREER Award CCF-2145474, and Jelani Nelson's ONR grant N00014-18-1-2562.} \and Hongxun Wu \thanks{Department of EECS, UC Berkeley. Email: wuhx@berkeley.edu. Supported by Avishay Tal's Sloan Research Fellowship, NSF CAREER Award CCF-2145474, and Jelani Nelson's ONR grant N00014-18-1-2562.} \and Junzhao Yang \thanks{IIIS, Tsinghua University. Email: yang-jz20@mails.tsinghua.edu.cn}}
\begin{document}

\maketitle

\begin{abstract}
We study the cost of parallelizing weak-to-strong boosting algorithms for learning, following the recent work of Karbasi and Larsen. Our main results are two-fold: 

\begin{itemize}
    \item First, we prove a tight lower bound, showing that even ``slight'' parallelization of boosting requires an exponential blow-up in the complexity of training. 

    Specifically, let $\gamma$ be the weak learner's advantage over random guessing. The famous \textsc{AdaBoost} algorithm produces an accurate hypothesis by interacting with the weak learner for $\tilde{O}(1 / \gamma^2)$\footnote{In this paper, we use $\tilde{O}$ to hide the terms logarithmic in $\gamma$ and $m$ where $m$ is the size of the training set.} rounds where each round runs in polynomial time.
    
    Karbasi and Larsen showed that ``significant'' parallelization must incur exponential blow-up: Any boosting algorithm either interacts with the weak learner for $\Omega(1 / \gamma)$ rounds or incurs an $\exp(d / \gamma)$ blow-up in the complexity of training, where $d$ is the VC dimension of the hypothesis class. We close the gap by showing that any boosting algorithm either has $\Omega(1 / \gamma^2)$ rounds of interaction or incurs a smaller exponential blow-up of $\exp(d)$.
    
    \item Complementing our lower bound, we show that there exists a boosting algorithm using $\tilde{O}(1/(t \gamma^2))$ rounds, and only suffer a blow-up of $\exp(d \cdot t^2)$. 

    Plugging in $t = \omega(1)$, this shows that the smaller blow-up in our lower bound is tight. More interestingly, this provides the first trade-off between the parallelism and the total work required for boosting.
\end{itemize}

Our lower bound follows from a novel interpretation of parallel boosting as a variant of ``coin game''. The upper bound is inspired by the ``bagging'' technique in machine learning and draws a connection to differential privacy.
\end{abstract}

\section{Introduction}

Boosting is one of the most important contributions from the theory to the practice of machine learning. In the 1980s, Kearns and Valiant \cite{kearns1988thoughts, kearns1989crytographic} raised the fundamental question of boosting: Can every learning algorithm that, given any distribution, outputs a classifier with slightly nontrivial accuracy (i.e., weak learner) be ``boosted'' into an algorithm that outputs a classifier with arbitrarily high accuracy (i.e., strong learner)? This question was answered positively by Schapire \cite{schapire1990strength}. Since then, boosting has found numerous applications in different fields of machine learning \cite{abernethy2021multiclass, brukhim2022boosting, shen2022federated}. Boosting frameworks such as XGBoost \cite{chen2016xgboost, chen2015xgboost}, LightGBM \cite{ke2017lightgbm} have been popular tools in machine learning practice. 

The classical boosting algorithm \textsc{AdaBoost}, developed by Freund and Schapire \cite{freund1997decision}, follows an elegant strategy consisting of many adaptive rounds. In the initial round, it runs the weak learner on the uniform distribution over the training set to obtain the first classifier. For every following round, it adjusts the distribution by increasing the mass on the data points where the last classifier makes a mistake. Then, it obtains the next classifier by running the weak learner on this new distribution. Intuitively, this forces the weak learner to ``focus'' on correcting its previous mistakes. After sufficiently many rounds, \textsc{AdaBoost} aggregates the opinion of all classifiers by taking a majority vote. If the hypotheses in each round have an advantage $\gamma$ over random guess, $\tilde{O}(1 / \gamma^2)$ rounds are sufficient for the aggregated classifier to achieve a $99\%$ accuracy.

The power of the seemingly simple strategy comes from two sources: (1) running the weak learner on many different distributions and (2) the adaptivity in choosing these distributions. This is also a shared characteristic of modern gradient boosters \cite{chen2015xgboost,friedman2001greedy,ke2017lightgbm}. However, adaptivity comes with a high cost. In reality, weak learners are usually implemented by training the model over the given distribution. One call to the week learner could take days (especially for many modern deep neural networks). The adaptivity makes the algorithm inherently sequential: no matter how many computational resources are available, boosting will always blow up the already-long training time multiplicatively.

\paragraph*{Impossibility of significant parallelization.} Recently, Karbasi and Larsen \cite{karbasi2023impossibility} proved that a significant parallelization is impossible, showing that certain adaptivity is necessary for boosting. 

To set up minimal notation, a $\gamma$-weak learner is a weak learner that always outputs classifiers with at least $\gamma$ advantage over the random guess. We say a weak learner uses a concept class $\mathcal H$ if its output is always a concept in $\mathcal H$. Then, the main theorem of Karbasi and Larsen \cite{karbasi2023impossibility} reads: 

\begin{theorem}[Special Case of Theorem 1,  \cite{karbasi2023impossibility}, Informally Rephrased] \label{thm:kasper}
There is a universal constant $\alpha > 0$ such that the following is true for any weak-to-strong learner (boosting algorithm) $A$. Suppose $A$ achieves $0.99$ accuracy with every valid $\gamma$-weak ($0 < \gamma < \alpha$) learner $\mathcal W$ that uses a concept set of VC dimension $d$. Then, either $A$ interacts with the weak learner for at least $p \ge 1/\gamma$ rounds, or $A$ makes at least $t \geq \min(\tilde{\Omega}(\exp(d / \gamma)), \exp(\exp(d)))$ oracle calls to the weak learner in total.\footnote{In their paper, they proved that the loss is at least $\ell \geq \exp(-O(p \max(\gamma, \ln(tp) \gamma^2 / d)))$ when $t \leq \exp(\exp(d))$ Plug in $\ell \leq 0.01$ and $p \leq 1/ \gamma$, one gets that $t \geq \min(\tilde{\Omega}(\exp(d / \gamma)), \exp(\exp(d)))$ here.}
\end{theorem}

As AdaBoost takes $\tilde{O}(1 / \gamma^2)$ rounds, this shows that any algorithm that quadratically parallelizes AdaBoost must make $\exp(d / \gamma)$ calls in each round. The authors of \cite{karbasi2023impossibility} concluded that \emph{``The classic algorithms, such as AdaBoost, use $O(\gamma^{-2} \ln m)$ rounds.\footnote{Again, here $m$ is the size of the data set.} Thus it is conceivable that boosting can be somewhat parallelized. We leave this as an exciting direction for future research.''}, which still leaves hope for some mild parallelization of boosting.

\subsection{Our Results}

\paragraph*{Impossibility of slight parallelization} Answering this hope in the negative, we prove that slight parallelization is impossible without an exponential number of calls to the weak learner each round. 

We note that, in the prior work of Long and Servedio~\cite{NIPS2011_b7ee6f5f}, they studied a more restrictive class of parallel boosting algorithms. For this class of boosters, they proved a stronger lower bound saying that $\Omega(1 / \gamma^2)$ rounds are necessary, even with an unbounded number of calls each round. (See \Cref{sec:related} for a more detailed discussion.) 

\begin{theorem} [Special Case of \Cref{theo:main-lower-bound}]  \label{thm:us}
There is a universal constant $\alpha > 0$ such that the following is true for any weak-to-strong learner (boosting algorithm) $A$. Suppose $A$ achieves $0.99$ accuracy with every valid $\gamma$-weak ($0 < \gamma < \alpha$) learner $\mathcal W$ that uses a concept set of VC dimension $d$. Then, either $A$ interacts with the weak learner for at least $p \ge 1/\gamma^2$ rounds, or $A$ makes at least $t\ge \exp(\Omega(d))$ oracle calls to the weak learner in total.
\end{theorem}

We proved this result via a connection to the coin problem, which we will discuss in detail in \Cref{sec:tech}.

\paragraph*{Tradeoff between parallelisms and total work.} Comparing \Cref{thm:kasper} and \Cref{thm:us}, one may notice that the lower bounds for $t$ are $\min(\exp(d / \gamma), \exp(\exp(d)))$ and $\exp(d)$ respectively. The $\exp(d)$ bound in \Cref{thm:us} seems weaker. A priori, it is not clear whether this is a technicality of our proof or an inherent nature of the problem. 

To answer this, we accompany our lower bound with a new boosting algorithm with fewer rounds of interaction, which shows that \Cref{thm:us} is nearly tight regarding the dependence on $\gamma$. 

\begin{theorem} [Informal version of \Cref{theo:upper-bound-formal}]\label{thm:upper-bound}
There exists a boosting algorithm, that boosts any $\gamma$-weak learner $\calW$ with a hypothesis class of VC dimension $d$ to a strong learner (with accuracy $0.99$), using $\tilde{O}(1 / (\gamma^2 R))$ rounds of interaction and $\exp(\tilde{O}(d R^2))$ calls to $\calW$ per round.

Specifically, when $R = \omega(1)$ and $R \ll \log (1 / \gamma)$, it uses $o(1 / \gamma^2)$ rounds and $\exp(\tilde{O}(d))$ queries per round.
\end{theorem}

In the work of Karbasi and Larsen \cite{karbasi2023impossibility}, they gave a one-round boosting algorithm by essentially enumerating all sparsely supported distributions and training weak learners on them in parallel. In comparison, our few-round boosting algorithm is inspired by a practical technique in machine learning, bagging (also known as bootstrap aggregation). Interestingly, similar ideas of using bagging to help parallelization of boosting have been experimentally explored by machine learning researchers (e.g., Figure 4 of \cite{yu2001parallelizing}, Algorithm 2 of \cite{palit2011scalable}, Algorithm 2 of \cite{lozano2005algorithms}). They gave experimental evidence that these algorithms of similar spirits work well in practice. To the best of our knowledge, our work is the first to give a theoretical analysis of such parallelization. Our analysis uses differential privacy tools, which we will discuss in more detail in \Cref{sec:tech}.

We also note that the idea from \Cref{thm:us} can be adapted to give a smooth trade-off on the lower bound side. 

\begin{theorem} [Informal version of \Cref{theo:trade-off-lower-bound}] \label{thm:lowerbound} 
There is a universal constant $\alpha > 0$ such that the following is true for any weak-to-strong learner (boosting algorithm) $A$. Suppose $A$ achieves $0.99$ accuracy with every valid $\gamma$-weak ($0 < \gamma < \alpha$) learner $\mathcal W$ that uses a concept set of VC dimension $d$. For every parameter $R\le \frac{1}{\gamma^2}$, either $A$ interacts with the weak learner for at least $p \ge \frac{1}{R}\frac{1}{\gamma^2}$ rounds, or $A$ makes at least $t\ge \exp(\Omega(d R))$ oracle calls to the weak learner in total.
\end{theorem}

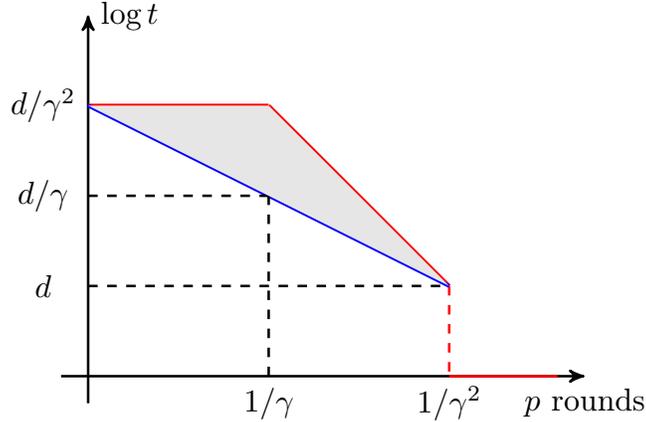
\begin{figure}[h]
\centering
\scalebox{1.2}{
\begin{tikzpicture}[
    thick,
    >=stealth',
    dot/.style = {
      draw,
      fill = white,
      circle,
      inner sep = 0pt,
      minimum size = 4pt
    }
  ]
  \coordinate (O) at (0,0);
  \draw[->] (-0.3,0) -- (5.5,0) coordinate[label = {below:\small $p$ rounds}] (xmax);
  \draw[->] (0,-0.3) -- (0,4) coordinate[label = {right:\small $\log t$}] (ymax);
  \draw [red, very thick] (0,3) -- (2,3);
  \draw [red, very thick] (2,3) -- (4,1);
  \draw [dashed, red] (4,1) -- (4,0);
  \draw [red, thick] (4,0) -- (5.2,0);
  \draw [blue, very thick] (0,3) -- (4,1);
  \node at (2,-0.3) {\small $1 / \gamma$};
  \node at (4,-0.3) {\small $1 / \gamma^2$};
  \draw [dashed] (2,0) -- (2,2);
  \draw [dashed] (0,2) -- (2,2);
  \draw [dashed] (0,1) -- (4,1);
  \node at (-0.5,3) {\small $d / \gamma^2$};
  \node at (-0.5,2) {\small $d / \gamma$};
  \node at (-0.5,1) {\small $d$};
  \fill[gray!20] (0,3) -- (2,3) -- (4,1);
\end{tikzpicture}}
\caption{Tradeoff between rounds of interaction $p$ and number of parallel queries in a single round $t$ (from \Cref{thm:upper-bound} and \Cref{thm:lowerbound} (ignoring all the log factors)). The red line is the upper bound and blue line is the lower bound. There is a phase transition when $p \approx 1 / \gamma^2$. The gray area indicates the current gap in the upper and lower bounds.} 
\label{fig:tradeoff}
\end{figure}

Note when $p = 1 / \gamma$, \Cref{thm:lowerbound} gives a lower bound of $t \geq \exp(d / \gamma)$ which improves over the $t \geq \min(\exp(d / \gamma), \exp(\exp(d)))$ lower bound in \Cref{thm:kasper} by Karbasi and Larsen \cite{karbasi2023impossibility}, showing that the $\exp(\exp(d))$ term is merely an artifact of the previous approach. This is a bonus from our coin problem approach. 

All together, our knowledge of the cost for parallelizing boosting can be summarized as \Cref{fig:tradeoff}. Closing the gap between the lower and upper bound for the $p\approx \frac{1}{\gamma}$ regime remains an interesting open problem.

\subsection{Related Works} \label{sec:related}

\paragraph*{Prior works on Parallel Boosting.} 

Numerous works in literature have studied the amount of resources required for boosting. The earlier work of Freund \cite{Freund95} showed that a total number of $\Omega(1/\gamma^2)$ calls to the weak learner is required. An important prior work of Long and Servedio \cite{NIPS2011_b7ee6f5f} considered the complexity of parallel boosting and proved that $\widetilde{\Omega}(1/\gamma^2)$ rounds of interaction with the weak learner is required, \emph{regardless} of the number of calls in each round. While this appears quantitatively stronger than our results, we note that the model considered in their work puts a significant restriction on the boosting algorithms, which makes the results weaker. 

Specifically, they assume that the query distributions of the booster is always derived through ``filtering''. Namely, it can only adjust the weight of any input $x\in X$ based on the classifications of $x$ given by the hypotheses from the weak learner in previous rounds. Under this assumption, they show that even an oblivious weak learner, whose output is independent of the query distribution, is sufficient to ``fool'' the booster. In contrast, our approach does not bind the booster to such constraints. Because of the assumption, their model cannot capture the boosting algorithms that query randomized distributions, such as our algorithms or the algorithms of Karbasi and Larsen \cite{karbasi2023impossibility}. Note that the ability to query randomized distributions is the key feature that allows us to achieve a tradeoff between query and round complexity and break their lower bound.

On the algorithm side, there are several boosting algorithms that make parallel calls to the weak learner, including the literature on boosting using decision tree learning \cite{kearns1996boosting} and branching programs \cite{mansour2002boosting, kalai2003boosting, long2005martingale, long2008adaptive}. Despite making parallel queries, all these algorithms still need $\Omega(1/\gamma^2)$ rounds of interaction with the weak learner. This is inherent because they are ``filtering'' algorithms, which are subject to the aforementioned lower bound by Long and Servedio \cite{NIPS2011_b7ee6f5f}.

\paragraph*{Differential Privacy and Bagging.} The idea of taking subsamples, performing computation, and finally aggregating the results is perhaps ubiquitous in computer science. For the machine learning side, the bagging technique has been investigated by numerous empirical and theoretical works and has a rich literature. For one recent example, it was shown in \cite{Larsen23-bagging-optimal} that the bagging technique, coupled with the Empirical Risk Minimization algorithm, gives a PAC-learning algorithm with optimal sample complexity. For the differential privacy side, the so-called ``privacy amplification by sampling'' \cite{BalleBG18-privacy-subsampling} and the ``sample-and-aggregate'' \cite{NissimRS07-smooth-sampling,ThakurtaS13-sample-aggregate} framework have been the main workhorses behind many exciting developments.

Our work contributes a new perspective on the power of subsampling in machine learning through an information-theoretical technique first developed and popularized in the Differential Privacy community, namely the advanced composition theorem. In our work, the privacy of the data set is not a primary concern, and our result is not directly comparable to the line of DP works. We hope that our work can inspire further investigation into the power of sample-and-aggregation and its interplay with ML and DP. 

\subsection{Our Techniques} \label{sec:tech}

\paragraph*{Review of the previous lower bound.} Before introducing our construction, it is instructive to review the previous lower bound of \Cref{thm:kasper}, which follows from a recursive set hiding structure. The main idea of \cite{karbasi2023impossibility} is as follows: Let $X_0$ be the set of all data points, and say we finally care about the accuracy on uniform distribution over $X_0$. The weak learner samples a chain of sets $X_0 \supset X_1 \supset X_2 \supset X_3 \supset \cdots \supset X_{p}$ with $|X_i| = (1 - 2\gamma) |X_{i - 1}|$ uniformly at random. Then, a random hypothesis $c: X_0\to \{\pm 1\}$ is drawn as the ground truth, and the goal of the strong learner is to approximate $c$ as well as possible. Roughly speaking, for the $i$-th round of queries, the weak learner will answer randomly within $X_i$ and truthfully outside of it (on $X_0 \setminus X_i$). The idea of the construction is intuitive: before the $i$-th round, the boosting algorithm knows nothing about $X_i$ and the correct labels of data points in $X_{i - 1}$. After round $i$, the boosting algorithm may learn the set $X_i$, but it has no information about the correct labels of data points in $X_i$. 

Since the boosting algorithm does not know what $X_i$ is before round $i$, ideally, the distribution $\calD$ it picks in the round $i$ should have $1 - 2 \gamma$ mass within $X_i$ and $2\gamma$ mass in $X_0 \setminus X_i$. Hence, one should expect it to have accuracy roughly $(1 - 2\gamma) \cdot \frac{1}{2} + 2 \gamma \cdot 1 = \frac{1}{2} + \gamma$. This makes it a valid weak learner. Taking $p = \Theta(1 / \gamma)$, one can make sure that $|X_{p}| \geq 0.01|X_0|$ so that the boosting algorithm, which has no information about correct labels in $|X_i|$, cannot have accuracy better than $0.99$. 

The reason why this approach stuck at $1 / \gamma$ rounds is that, in order to ensure a $\gamma$-advantage over random guess, the weak learner can conceal at most $1 - 2\gamma$ fraction of the correct labels, then naturally, after roughly $1 / \gamma$ rounds, most correct labels are released to the boosting algorithm. 

\medskip\noindent\textbf{Connection to Coin Problem.}
We overcome this barrier by drawing an interesting connection to the coin problem. We think of the weak learner as having $|X_0|$ many random coins. For any $x \in X_0$, if the correct label is $1$, then the corresponding coin is $\gamma$-biased towards head, and if the true label is $0$, the coin is $\gamma$-biased towards tail. Loosely speaking, in each round $i$, the weak learner tosses every coin \emph{exactly once} and makes up a hypothesis $h^{(i)}:X_0\to \{0,1\}$ that records the toss results. In the ideal case, every distribution $\cal D$ picked by the boosting algorithm is spread-out enough. Then by concentration among the coin tosses, with exponentially small probability, each query $\cal D$ can be answered by $h^{(i)}$ with an $\Omega(\gamma)$ advantage.

In this way, the weak learner gives away only $O(\gamma^2)$ bits of information on the correct label on \textit{every} data point, while the prior construction reveals a $2\gamma$ fraction of correct labels in every round. Consequently, we are able to provide a tight $\Omega(1 / \gamma^2)$ lower bound on the number of rounds.

We note that in this ideal case, this coin problem approach is similar to that of the previous lower bounds \cite{Freund95, NIPS2011_b7ee6f5f}. The difference lies in the case that the query distribution $\cal D$ is sparse (i.e., not spread-out enough). We construct a family of hypotheses ${\cal H}^{(i)}$ consisting of $h^{(i)}$ and a small number of random hypotheses. (This construction has the same spirit as the approach by Karbasi and Larsen \cite{karbasi2023impossibility}.) In the case where $\cal D$ is sparse, we show that the weak learner can answer the query with $\Omega(\gamma)$ advantage using one of the random hypotheses in ${\cal H}^{(i)}$. We also slightly simplified their approach by a more fine-grained division between the sparse and spread-out cases.

\paragraph*{Connection to Bagging and Differential Privacy} 
Bagging is the technique of creating many data sets by independently sampling the training set. Our algorithm itself is very simple (See \Cref{algo:intro}). The key idea is to simulate a weak learner with a small set of hypotheses $\calH_k$ obtained by bagging. 

\begin{algorithm}[h]
\DontPrintSemicolon
    \caption{Sketch of the Parallel Boosting Algorithm}\label{algo:intro}
    Initialize the distribution $D$ to be uniform over the training set.\\
    \For{$k \gets 0$ \KwTo $\lceil K/R \rceil-1$}{
        \tcp{Bagging Step.}
        \For{$q \gets 1$ \KwTo $Q$}{    
            Subsample $n$ elements from the training set according to $D$.\\
            $D_{k,q} \gets$ the uniform distribution over these $n$ elements.\\
        }
        $\calH_k \gets$ hypotheses obtained by running weak learner on $D_{k,q}$ for each $q \in [Q]$ separately. \label{line:call}\\
        \tcp{Boosting Step.}
        \For{$r \gets kR$ \KwTo $(k+1)R - 1$}{
            $h_r \gets \arg \min_{h \in \calH_k} \Loss_D(h)$.  \tcp{Here $h_r$ is easy to find because $|\calH_k| \leq Q$.} 
            Update $D$ as if the weak learner has outputted $h_r$ on $D$. \label{line:upd}
        }
    }
    \Return Aggregation of all $K$ different $h_r$'s. 
\end{algorithm}

For its analysis, we want to argue that bagging helps the algorithm be more robust against a few distribution updates. This is when \emph{advanced composition}, a tool from differential privacy (see, e.g.~\cite{dwork2010boosting}), comes in handy. Recall that bagging generates many new datasets by subsampling. Loosely speaking, we will use it to show that the mechanism for generating one such new dataset satisfies approximate DP.  (Looking ahead, the input to the mechanism will be the accumulated loss on every data point.) Hence intuitively, for just a few updates, they will not affect the result of bagging much. With some effort, we can show that at least one of the hypotheses in $\calH_{k}$  will have $\gamma$ advantage over random guess. This guarantees that at \Cref{line:upd}, we successfully simulated a valid weak learner by our parallel calls to $\calW$ at \Cref{line:call}.

\section{Preliminaries}

\paragraph{Notation.} $\R_+^n$ denotes the set of all non-negative real numbers. $\Ber(p)$ denotes the Bernoulli distribution with parameter $p$. Namely, $X\in \Ber(p)$ takes value $1$ with probability $p$ and value $0$ otherwise. For a finite set $A$, we use $x\sim A$ to denote a random variable $x$ distributed uniformly at random from $A$. We usually use $\calX$ to denote the universe. Then, for every subset $S\subseteq \calX$, we denote by $\overline{S}\coloneqq \calX\setminus S$ the complementary set of $S$.

\subsection{PAC Learning and Boosting}

This paper studies \emph{boosting algorithms}. Namely, algorithms that make oracle queries to a \emph{weak learner} and convert it into a \emph{strong learner}.

\paragraph*{PAC Learning.} We consider the task of \emph{binary classification} and employ the PAC learning framework, which we briefly review now. We use $\calX$ to denote a finite domain of inputs. The binary labels are denoted by $\{ \pm 1\}$. Let $\calH\subseteq \{ h : \calX\to \{\pm 1\} \}$ be a class of hypotheses. The VC dimension of $\calH$, denoted by $\VCDim(\calH)$, is the maximum $d\in \mathbb{N}$ such that there exists a subset set $S\subseteq \calX$ of size $d$ that is \emph{shattered}\footnote{This means the projection of $\mathcal{H}$ onto S consists of all $2^{|S|}$ possible hypotheses.} by $\calH$. A learning task is usually described by a distribution $\calD$ over $\mathcal{X}\times \{\pm 1\}$, and the goal of learning is to find a hypothesis $\hat{h}$ (not necessarily from $\mathcal{H}$) minimizing the loss function on $\calD$, which is defined as $\Loss_\calD(\hat{h})\coloneqq \Ex_{(x,y)\sim \calD}[h(x)\ne y]$. The distribution $\mathcal{D}$ is called \emph{realizable}, if there is a hypothesis $h\in \mathcal{H}$ such that $\Loss_\calD(h) = 0$.

\medskip\noindent\textbf{Weak learner.} Let $\gamma>0$ be a parameter. For our purpose, a $\gamma$-\emph{weak learner} with hypothesis set $\mathcal{H}$ is an oracle that can be queried with distributions supported on $\calX\times \{\pm 1\}$. Given a query distribution $\calD$, the weak learner returns a hypothesis $h\in \mathcal{H}$ such that
\[
\Loss_\calD(h)\coloneqq \Ex_{(x,y)\sim \calD}[ h(x) \ne y ] \le \frac{1}{2}-\gamma.
\]
If there are multiple $h\in \mathcal{H}$ with loss bounded by $\frac{1}{2}-\gamma$, the weak learner can return an arbitrary one. The weak learner may declare ``failure'' if there is no valid hypothesis. We call the oracle a ``weak learner'' because the learner produces hypotheses ``slightly'' better than random guess (note that we think of $\gamma$ as small).

\medskip\noindent\textbf{Strong learner and Boosting.} A strong learner is a learning algorithm that, given $n$ i.i.d. samples from a \emph{realizable} distribution $\calD$, with probability $1-\delta$, produces a hypothesis $\hat{h}$ with $\Loss_\calD(\hat{h})\le \eps$. Usually, in a strong learning algorithm, the accuracy $\eps$ and confidence $\delta$ can be arbitrarily small with only a $\poly(1/\varepsilon, \log(1/\delta))$ overhead in the sample and computational complexity.

In the boosting framework, the goal is to design a general weak-to-strong learner (a.k.a.~a boosting algorithm) \emph{without} the knowledge of the hypothesis class $\mathcal{H}$. Instead, the boosting algorithm is given oracle access to a weak learner for $\mathcal{H}$, and the boosting procedure should work well for every valid weak learner oracle.

More precisely, let $\calA$ be a boosting algorithm. There is a hypothesis class $\calH$ and a realizable distribution $\calD$, both unknown to the boosting algorithm $\calA$. First, $\calA$ receives $m$ i.i.d.~samples from $\calD$, denoted by $ = \{(x_i,y_i)\}_{i\in [m]}$. Then, $\calA$ interacts with a weak learner $\calW$ of $\calH$ for $p\ge 1$ rounds. In each round, $\calA$ sends a set of at most $t$ queries to $\calW$, where each query is a distribution $\calD'$ over $S$. The weak learner then answers all the queries in parallel. Namely, for each $\calD'$, the weak learner reports a hypothesis $h\in \calH$ such that $\loss_{h}(\calD') \le \frac{1}{2} - \gamma$. Finally, after the $p$ rounds of interaction, $\calA$ returns a hypothesis $h^*$ (not necessarily from $\calH$) that tries to minimize the loss $\loss_{\calD}(h^*)$. We use $\calA^{\calW}(S)$ to denote the output of $\calA$ with oracle $\calW$ and input data set $S$.

We can see that the complexity of the boosting algorithm is parameterized by three parameters $m,p,t$. We call $m$ the \emph{sample complexity} of $\calA$, and say the algorithm has \emph{parallel complexity} $(p,t)$.

\section{Lower Bound Against Slight Parallelization}

Our negative result shows that, given oracle access to a $\gamma$-weak learner, any boosting algorithm with $o(\frac{1}{\gamma^2})$ rounds of interaction cannot learn the target hypothesis with $1-o(1)$ accuracy unless it makes exponential (in VC dimension) many calls to the weak learner in each round. We formalize our claim into the following main theorem.  

\begin{theorem}\label{main-theorem}\label{theo:main-lower-bound}
There is a universal constant $\clb > 0$ for which the following is true. For every $0 < \gamma < 1/2$, every $d,m\ge 1$, let $\calA$ be a boosting algorithm that uses $m$ samples and has parallel complexity $(p, t)$ where $p\le \min(\frac{\clb d}{\gamma^2}, \exp(\clb d))$ and $t\le \exp(\clb d)$. Then, there exists a domain $\calX$ of size $2m$, a hypothesis class $\calH\subseteq \{\pm 1\}^{|\calX|}$ of VC dimension $\frac{d}{\clb}$, a realizable distribution over $\calX\times \{\pm 1\}$, and a $\gamma$-weak learner $\weakl$ for $\calH$ such that
\begin{align} \label{eq:expected-loss-in-lower-bound}
\Ex_{S\sim \calD^{m},\calA}[\loss_\calD(\strl^{\weakl}(S))] \ge \exp(-O(p \gamma^2 + 1)).
\end{align}
\end{theorem}

The rest of the section is devoted to the proof of \Cref{theo:main-lower-bound}.

\subsection{Construction of Hard Instances} \label{sec:lower-bound-construct}

In this subsection, we fix parameters $d,m$ in \Cref{theo:main-lower-bound}, and describe our construction of the hard hypotheses class $\calH$ and weak learner $\calW$. We will give a randomized construction that is \emph{independent} of the boosting algorithm. Then, for any fixed boosting algorithm $\calA$, we will show that our construction incurs a noticeable loss in expectation, which implies that there exists an instantiation of the construction on which $\calA$ fails to boost. %

\paragraph*{Defining the learning task.} Let $m > 0$ be the number of samples requested by the strong learner, which can be arbitrarily large (e.g., it can be as large as $d^{d^d}$). We aim to construct a learning task with a hypothesis class of VC dimension $d$, on which the strong learner fails even with $m$ samples.

Now, given $m$, we define the input domain to be $\calX = \{x_1, x_2, \dots, x_{2m}\}$. We sample a concept $c$ over $\calX$ by setting each $c(x_i)$ to $\pm 1$ randomly and independently. In the lower bound proof, we will think of $c$ as the ``ground truth'', which the strong learner attempts to learn. Hence, we define the target distribution $\mathcal{D}$ as the uniform distribution over $\{(x_i,c(x_i))\}_{1\le i\le 2m}$. 

\paragraph*{Weak learner and hypothesis class.} We now construct the hypothesis class $\mathcal{H}$. We conduct the construction in $p$ stages. In each stage $i\in [p]$, we first create a hypothesis class $\mathcal{H}^{(i)}$ of size $2^{O(d)}$ where big-$O$ hides an absolute constant. The construction of $\calH^{(i)}$ is as follows.

\begin{itemize}
    \item First, we sample a hypothesis $a^{(i)} \in \{1, -1\}^{2m}$ by drawing each $a^{(i)}(x_j) \sim (-1)^{\Ber(1/2 - \Cbias \cdot \gamma \cdot c(x_j))}$ independently. 
    Considering $c$ as the ``ground truth'', the hypothesis $a^{(i)}$ gives an expected advantage of $\Cbias\gamma$ over random guess.
    \item For a parameter $\hatd = \Theta(d)$ (Here, $\Theta$ hides an absolute constant that will be specified later), draw $r^{(i)}_1,\dots, r^{(i)}_{2^{\hatd}}$ where each $r^{(i)}_{j} \sim \{\pm 1\}^{2m}$. Namely, we draw $2^{\hatd}$ hypotheses uniformly at random.
    \item Finally, we define $\calH^{(i)} = \{a^{(i)}\} \cup \{r^{(i)}_j : 1\le j\le 2^{\hatd}\}$.
\end{itemize}

We set our final hypothesis class as $\calH = \{c\} \cup \bigcup_{i=1}^{p} \calH^{(i)} $. It is easy to see that $|\calH|\le 2^{O(d)}$ and consequently $\VCDim(\calH)\le O(d)$. 

Next, we define the weak learner $\calW$. We sort all hypotheses in the order of $\calH^{(1)},\dots, \calH^{(p)}, \{c\}$ (the order inside a sub-class can be arbitrary). Then, on a given query $\calD$, $\calW$ outputs the \emph{first} hypothesis $h$ from the list such that $\loss_{\calD}(h) \le \frac{1}{2}-\gamma$. This defines a valid weak learner for $\calH$.

\paragraph*{Proof Outline.} We present the overall proof structure and state two key claims. Assuming them, we quickly finish the proof of \Cref{theo:main-lower-bound}. We prove the claims in subsequent subsections.

Our first claim is that, with high probability, the queries made by $\calA$ in the first $i$ rounds can be answered by hypotheses from $\calH^{(1)},\dots, \calH^{(i)}$ \emph{only}.

\begin{claim} \label{low-fail-prob} \label{claim:low-fail-prob} 
Let $\calA$ be a boosting algorithm. For every $i\le p$ and every fixed realization of $c\sim \{\pm 1\}^{2m}$ and $S\sim \calD^{m}$, with probability $1-i\cdot \exp(-\Omega(d))$ (over the construction of $\calH$ and the interaction between $\calA$ and $\calW$), the queries made by $\calA$ in the first $i$ rounds can be answered by hypotheses from $\calH^{(1)},\dots,\calH^{(i)}$.
\end{claim}

\Cref{claim:low-fail-prob} is proved in \Cref{sec:low-fail}. Assuming it for now, the following corollary is immediate.

\begin{corollary} \label{simulation-claim} \label{claim:simulation}
    With probability $1-\exp(-\Omega(d))$, the interaction between $\strl$ and $\weakl$ can be simulated given $S,\calH^{(1)},\dots, \calH^{(p)}$. In particular, the simulation does not require the knowledge of $c$.
\end{corollary}

\begin{proof}
    By Claim~\ref{claim:low-fail-prob} and the condition that $p \le \exp(\clb \cdot d)$ where $\clb>0$ is sufficiently small, we have with probability $1-\exp(-\Omega(d))$ that the queries of $\calA$ can be answered with hypotheses from $\calH^{(1)},\dots, \calH^{(p)}$ only. 
    
    We now describe the simulation. We can start by simulating $\calA$ given the data $S$. In each of the $p$ rounds, $\calA$ issues a set of queries. Under the aforementioned event, each query can be answered by some hypothesis from $\calH\setminus \{c\}$, which means we can compute the responses of the weak learner given $\calH \setminus \{c\}$. Finally, after the $p$ rounds of interaction, we finish by simulating the remaining pieces of $\calA$ and reporting its output.
\end{proof}

Next, we would like to show that, given only $S$ and $\calH \setminus \{c\}$, the learner cannot produce a hypothesis with low generalization error.

\begin{claim} \label{loss-lowerbound} \label{claim:loss-lowerbound}
    Let $\calB = \calB(S,\calH^{(1)},\dots, \calH^{(p)})$ be an arbitrary aggregation procedure that outputs a hypothesis $\hat{h}$. We have
    \[
    \Ex_{c,S,\calH}[\loss_{\calD}(\hat{h} = \calB(S,\calH^{(1)},\dots,\calH^{(p)}))] \ge \exp(-O(p\gamma^2 + 1)).
    \]
\end{claim}

\Cref{loss-lowerbound} is proved in \Cref{sec:loss-lowerbound}. Assuming it, we prove our main theorem.
\paragraph*{Proof of \Cref{main-theorem}.}
Let $\calE$ be the event defined in \Cref{claim:simulation}. It follows that $\Pr[\overline{\calE}] \le \exp(-\Omega(d))$. Under the event $\calE$, the behavior of the boosting algorithm can be written as a procedure $\calB(S,\calH^{(1)},\dots,\calH^{(p)})$ that depends only on $S$ and $\calH\setminus \{c\}$. By \Cref{loss-lowerbound}, we obtain
\[
\begin{aligned}
&~~~~ \Ex_{c\sim\{\pm 1\}^{2m},S\sim\calD^{m},\calH, \calA} \left[\loss_{\calD}(\calA^{\calW}(S)) \right] \\
& \ge \Pr[\calE] \cdot \Ex[\loss_{\calD}(\calB(S,\calH^{(1)},\dots,\calH^{(p)})) \mid \calE] \\
& \ge \Ex[\mathbbm{1}_{\calE \text{ happens}} \cdot  \loss_{\calD}(\calB(S,\calH^{(1)},\dots,\calH^{(p)}))] \\
& \ge \Ex[\loss_{\calD}(\calB(S,\calH^{(1)},\dots,\calH^{(p)}))] - \Pr[\overline{\calE}] & \text{(loss is always bounded by $1$)}\\ 
& \ge \exp(-O(p\gamma^2 + 1)).
\end{aligned}
\]
This shows that a random construction of $\calH$ incurs a loss of $\exp(-O(p\gamma^2 + 1))$ in expectation. By averaging, there exists an instantiation of the construction satisfying the proposition of \Cref{theo:main-lower-bound}, which completes the proof.

\subsection{Proof of \Cref{low-fail-prob}} \label{sec:low-fail}

We prove \Cref{low-fail-prob} by inducting on $i\le p$. Assume the claim is true for the first $i-1$ rounds. Fix the data set $S$. Under the event that queries of the first $i-1$ rounds were answered using hypotheses from $\calH^{(1)},\dots,\calH^{(i-1)}$, the current set of at most $t$ queries, denoted by $\calD^{(i)}_1,\dots, \calD^{(i)}_{t'}$, are independent of $\calH^{(i)}$: they only depend on the internal randomness of $\calA$ after conditioning on $S$ and $\calD^{(j)}$ for $j < i$.

Let $\calD'$ be one query from $\{ \calD^{(i)}_1,\dots, \calD^{(i)}_{t'} \}$. We will show that the probability $\calD'$ fails to be answered by $\calH^{(j)}, j\le i$, is bounded by $\exp(-\Omega(d))$. Having established the claim, we may union bound over all $t' \le \exp(\clb \cdot d)$ queries and finish the proof for the $i$-th round. 

    Now, assume $\calD'$ is a query distribution that is independent of $\calH^{(i)}$. We will further divide the query into two cases. Namely the \emph{spread} case and the \emph{concentrated} case. Let $\athre > 0$ be a (large) constant to be specified later. We begin by formalizing the definition of spread distribution.

\begin{definition}[Spread distribution] 
Let $\calD'$ be a distribution over $\calX$. Sort elements in $\calX$ by their probability mass under $\calD'$ and let $w_{(i)}$ denote the $i$-th largest probability mass. We define the spreadness of $\calD'$ as
\begin{align*}
    F(\calD') = \sum_{i=1}^{d} w_{(i)} + \sqrt{d} \left( \sum_{i=d+1}^{|\calX|} w_{(i)}^2\right)^{1/2}.
\end{align*}
We say that $\calD'$ is \emph{spread} if $F(\calD')$ is smaller than $\athre\gamma$. Otherwise, we say $\calD'$ is \emph{concentrated}.
\end{definition}

\begin{remark}
    The idea of considering spread and concentrated queries separately is directly inspired by \cite{karbasi2023impossibility}. However, we use a different definition of ``spread'' than the prior work, which turns out to be much simpler and facilitate our proof significantly.
\end{remark}

\paragraph*{Concentrated Queries.} Suppose $\calD'$ is a concentrated distribution. In this case, we first observe that the expected loss of a truly random hypothesis $r \sim \{1, -1\}^{2m}$ is $\frac{1}{2}$. For concentrated query distribution $\calD'$, we will use an anti-concentration inequality to lower bound the probability that $\loss_{\calD'}(r) < 1/2 - \gamma$. To start, define a function $F' : \R_+^n \times \R_+ \to \R$ for every non-negative vector $w \in \R_+^n$ and real $t > 0$:
\begin{align*}
    F'(w, t) \coloneqq \sum_{i=1}^{\lfloor t^2 \rfloor } |w_{(i)}| + t \left( \sum_{j=\lfloor t^2 \rfloor +1}^{n} w_{(j)}^2\right)^{1/2}
\end{align*}
where $w_{(i)}$ denotes the $i$-th largest entry of $w$.

\begin{lemma}[\cite{montgomery1990distribution}]\label{theo:anti-concentration} There are universal constants $\alpha_1, \alpha_2>0$ such that the following holds: For any vector $w \in \R_+^n$ and $t > 0$, it holds that $\Pr_{x\sim\{\pm 1\}^n}[ \langle w, x \rangle > \alpha_1 F'(w, t)] \ge \alpha_2^{-1} \exp(-\alpha_2t^2)$. 
\end{lemma}

By our definition, note that $F(\calD') = F'(\calD', \sqrt{d})$. Setting $t = \sqrt{d}$ and applying \Cref{theo:anti-concentration}, we obtain
\begin{align*}
    \Pr[1/2 - \loss_{\calD'}(r) >  \gamma] \ge \Pr[\langle \calD', r \cdot c \rangle  > \alpha_1 F(\calD')] \ge \alpha_2^{-1}\exp(-\alpha_2d),
\end{align*}
provided that $\athre > \frac{2}{\alpha_1}$. Here, $r \cdot c$ denotes the pointwise product of $r$ and $c$. Setting $\hatd = \log_2 (\alpha_2\exp(\alpha_2 d)) + 2d = \Theta(d)$, the probability that exists one hypothesis $r \in \calH^{(i)}$ such that $r$ has $\gamma$ advantage over $D'$ is at least $1 - (1 - \alpha_2^{-1} \exp(-\alpha_2 d))^{2^{\hatd}} \ge 1 - \exp(-\exp(d))$.

\paragraph*{Spread Queries.} For a fixed round $i$, let $\calD'$ be a spread distribution queried in the $i$-th round. Recall we chose $a^{(i)}$ as a random hypothesis with an expected advantage of $4\Cbias\cdot \gamma$. 

Assume the claim is true for the first $i-1$ rounds (that is, all the queries in the first $i-1$ round were answered using hypotheses from $\calH^{(1)}\cup \dots \cup \calH^{(i-1)}$). Then, the query $\calD'$ is independent of $a^{(i)}$. We now argue that, for a spread query $\calD'$, the probability that $a^{(i)}$ fails to achieve $\gamma$ advantage on $\calD'$ is exponentially small.

For each $j\in [2m]$, let $y_j \coloneqq \calD'(x_j) \cdot \mathbbm{1}[a^{(i)}(x_j) \ne c(x_j)]$ be the random variable denoting the contribution to the loss from the input $x_j$. Let $B\subseteq [2m]$ be the set of indices of the $d$ largest probability mass. We observe that
\[
\Pr\left[\loss_{\calD}(a^{(i)}) > \frac{1}{2} - \gamma\right] = \Pr\left[ 
 \sum_{j\in B} y_j + \sum_{j\in \overline{B}} y_j > \frac{1}{2} - \gamma\right].
\]
By definition, with probability one, we have
\[
\sum_{j\in B} y_j \le \sum_{j\in B} \calD'(x_j) \le F(\calD') \le \athre \gamma.
\]
Therefore, it suffices to prove that
\[
\Pr\left[ \sum_{j\in \overline{B}} y_j > \frac{1}{2} - (\athre+1) \gamma \right] \le \exp(-d).
\]
We wish to apply Hoeffding's inequality. To begin with, we calculate:
\[
\left( \sum_{j\in \overline{B}} \calD'(x_j)^2 \right)^{1/2} \le \frac{F(\calD')}{\sqrt{d}} \le \frac{\athre \cdot\gamma}{\sqrt{d}}.
\]
Denote $S\coloneqq \sum_{j\in \overline{B}} y_j$. We have
\[
\Ex[S] \le \frac{1}{2} - \Cbias \gamma.
\]
Hence, choosing $\Cbias$ to be larger than $3\athre + 1$, we obtain
\[
\Pr[S > \frac{1}{2} - (\athre + 1)\gamma] \le \Pr[ S - \mathbf{E}[S] > 2\athre\cdot \gamma] \le \exp\left(\frac{(2\gamma)^2}{\sum_{j\in \overline{B}_j}(2\calD'(x_j))^2} \right) \le \exp(-d).
\]

To wrap up the two cases, given a query $\calD'$, being spread or not, with probability $1-\exp(-\Omega(d))$ over the sampling of $\calH^{(i)}$, there is an $h \in \calH^{(i)}$ such that $\loss_{\calD'} (h) < 1/2 - \gamma$, which means that $\calD'$ can be answered by a hypothesis from $\bigcup_{j\le i} \calH^{(j)}$. We conclude the proof by union-bounding over all queries in the $i$-th round.

\subsection{Proof of \Cref{loss-lowerbound}}\label{sec:loss-lowerbound}

In the coin problem, there is a hidden coin obeying either $\Ber(\frac{1+ \eps}{2})$ or $\Ber(\frac{1 - \eps}{2})$. The algorithm is given as input $n$ tosses $x_1, x_2, \dots, x_n$ of the coin, and is asked to distinguish these two cases. In our proof, we use the following information-theoretic lower bound for the coin problem.

\begin{lemma}[Coin problem. See e.g. Exercise 3.19, \cite{mohri2018foundations}.] \label{coin-lowerbound}
Let $\eps \in (0,1/2)$. For every (possibly randomized) algorithm $f:\{0,1\}^{n}\to \{\pm 1\}$, we have
\[
\Ex_{b\sim \{\pm 1\}} \left[ \Ex_{x\sim \Ber(\frac{1+b\cdot \eps}{2})^{n}}[f(x) \ne b] \right] \ge \exp(-O(\varepsilon^2 n + 1)).
\]
\end{lemma}

\begin{proofof}{\Cref{loss-lowerbound}} Let $\calB = \calB(S,\calH^{(1)},\dots, \calH^{(p)})$ be an arbitrary procedure that outputs a hypothesis. Let $S_0\subseteq \calX$ be the subset of inputs that appeared in $S$ (namely, $S_0$ is the projection of $S$ onto $\calX$). Conditioning on $S_0$, we observe that $\calB$ is essentially playing $2m - |S_0|$ independent instances of coin games in parallel. 

In more detail, for each $\calH^{(i)}$, all the random hypotheses from $\calH^{(i)}$ are independent of the ground truth $c$ (which means $\calB$ could have generated them by itself). Next, each $a^{(i)}$ is obtained by, for each input $x\in \calX$, flipping a biased coin from the distribution $\Ber(\frac{1}{2} + \Cbias\cdot \gamma\cdot c(x))$. Fix an input $x\in \calX\setminus S_0$. We further observe that the value of $a^{(i)}(x')$ for $x'\ne x$ is independent of $c(x)$. Thus, by \Cref{coin-lowerbound}, we obtain
\[
\Ex_{c\sim \{\pm 1\}^{2m}}\left[ \Pr_{\substack{ \calH^{(1)},\dots,\calH^{(p)}\\ h = \calB(S,\calH^{(1)},\dots, \calH^{(p)}) }} [h(x) \ne c(x) ]\right] \ge \exp(-O(\gamma^2 p + 1)).
\]

By linearity of expectation, we thus obtain
\[
\Ex_{\substack{ c, \calH^{(1)},\dots,\calH^{(p)}\\ h = \calB(S,\calH^{(1)},\dots, \calH^{(p)}) }} [ \Loss_{\calD}(h) ] \ge \frac{2m-|S_0|}{2m} \cdot \exp(-O(\gamma^2 p)) \ge \exp(-O(\gamma^2 p + 1)).
\]
Finally, averaging over $S_0$ finishes the proof.
\end{proofof}

\section{Trade-off between Parallelism and Total Work}

In this section, we study the round-query trade-off of parallel boosting. In \Cref{subsec:algo}, we will first present a boosting algorithm based on a variant \textsc{AdaBoost} and inspired by bagging. The key lemma in its analysis will be proved via differential privacy in \Cref{subsec:analysis}. Finally, we generalize our lower bound to give the round-query trade-off in \Cref{subsec:lb}.

\subsection{The Few Rounds Boosting Algorithm} \label{subsec:algo}

This subsection is devoted to the following theorem.

\begin{theorem} \label{theo:upper-bound-formal} For any data distribution $\calD$ over $\calX$, any unknown concept $c: X \to \{1, -1\}$, any $\gamma$-weak learner $\weakl$ that produces hypothesis class of VC dimension $d$, for any $1 \le R \le 1/2\gamma$, let training set $S \sim \calD^m$ and setting $m = \tilde{O}(d \gamma^{-2})$, there exists a (randomized) boosting algorithm $\calA$ satisfying $\Pr_{h \gets \calA}[\loss_\calD(h) < 0.1] \ge 0.9$, such that $\calA$ runs in $O(\gamma^{-2} \ln m / R)$ rounds and makes $ \exp(O(dR^2 )) \ln(1/\gamma)$ queries each round. 
\end{theorem}

\paragraph{Algorithm Description.} Our algorithm contains $K / R$ rounds. In the $k$-th round ($k \in [K/R]$), it contains a bagging step followed by a boosting step. 

\begin{itemize}
    \item Bagging step: It samples $Q$ different training sets of size $n$ from the current distribution $D_{kR}$. For each of these training sets, it makes a parallel call to the weak learner $\calW$. Let $\calH_k$ be the set of output hypotheses of the weak leaner calls. 
    \item Boosting step: It constructs a simulated weak learner $\calW'$ which, for any distribution, simply outputs the best-performing hypotheses among these $Q$ hypotheses in $\calH_k$. It then performs $R$ steps of \textsc{AdaBoost} by interacting $R$ rounds with simulated weak leaner $\calW'$. \textsc{AdaBoost} updates the current distribution $D_{kR}$ times, generating $D_{kR+ 1}, D_{kR + 2}, \dots, D_{(k + 1)R}$. 
\end{itemize}

\begin{algorithm}[h]
\DontPrintSemicolon
    \caption{Round-query trade-off boosting}
    \label{algo:round-query-boosting}
    \SetKwInOut{Input}{Input}\SetKwInOut{Output}{Output}
    \Input{Labelled Training Set $(S, c(S)) = (x_1, c(x_1)), (x_2, c(x_2)), \dots, (x_m, c(x_m))$, \\
    number of calls to weak learner per round $Q$}
    \Output{Aggregated voting classifier $\hat{h}$} 
    $w \gets \frac{1}{2}\ln(\frac{1/2+\gamma/4}{1/2-\gamma/4})$ \hspace{0.5cm} \tcp{Learning Rate}
    $n \gets c' d \gamma^{-2}$ \hspace{0.5cm} \\
    $\calD_0 \gets (\frac{1}{m}, \frac{1}{m}, \dots, \frac{1}{m})$ \\
    \For{$k \gets 0$ \KwTo $\lceil K/R \rceil-1$}{
        \tcp{Bagging Step} 
        $\calQ_k \gets \emptyset$ \\
        \For{$q \gets 1$ \KwTo $Q$}{
            Sample a multiset $T_{k,q} \sim (\calD_{kR})^n$ \label{line:D_kR}  \label{line:Q-samples}\\
            $D_{k,q} \gets $uniform distribution over $T_{k,q}$ \\
            $\calQ_k \gets \calQ_k \cup \{D_{k,q}\}$
        }
        $\calH_k \gets $ hypotheses from $\weakl$ after querying $\calQ_k$ in parallel\\
        \tcp{Boosting Step}
        \For{$r \gets kR$ \KwTo $\min((k+1)R, K)-1$}{
            \If{exists $h \in \calH_k$ s.t. $\loss_{\calD_r}(h) \le 1/2-\gamma/4$} { \label{Line:if} 
                $h_r \gets $ such $h$ that $\loss_{\calD_r}(h) \le 1/2-\gamma/4$  \label{line:D_r}\\ 
                \For{$i \gets 1$ \KwTo {m}} {
                    $\calD_{r+1}(i) \gets \calD_{r}(i) \exp(- c(x_i)h_r(x_i)\cdot w)$
                }
                Normalize $\calD_{r+1}$
            } \Else{
                \Return failed \label{line:fail}
            }
        }
    }

    \For{$x \in \calX$} {            
    $g(x) \gets \frac{1}{K}\sum_{k=0}^{K-1} h_k(x)$ \label{line:g-def}\\
    $\hat{h}(x) \gets \mathrm{sign}(g(x))$ \\
    }
    \Return $\hat{h}$
\end{algorithm}

Specifically, we will set $K = 16 \gamma^{-2} \ln m$ and $Q = \exp(16c'dR^2) \ln(1/\gamma)$ where $m$ is the size of the training set. See \Cref{algo:round-query-boosting} for the detailed pseudocode of this algorithm. 

\paragraph*{Analysis.} By the standard analysis of \textsc{AdaBoost}, as long as simulated $\calW'$ always succeeds in outputting a hypotheses with $\gamma$ advantage (i.e. \Cref{algo:round-query-boosting} never fails and reaches Line \Cref{line:fail}), the aggregated hypotheses $\hat{h}$ will satisfy $\Loss_D(h) < 0.1$ with good probability. Hence the key of the analysis is to prove the following lemma:

\begin{lemma}[Key lemma]\label{lemma:boosting-failure-prob}
There exists a universal constant $c' > 1$, such that for any $r \in [K]$, if we set $R \leq 1 / 2\gamma$, $Q \geq  \exp(16c'dR^2) \ln(1/\gamma)$, $K = 16 \gamma^{-2} \ln m$ and $m = \ln(1/\gamma) \ln(d/\gamma) c' d/\gamma^2$ in \Cref{algo:round-query-boosting}, the probability that $h$ does not exist for that $\calD_r$ at \Cref{Line:if} is less than $0.01/K$. 
\end{lemma}

We defer the proof of \Cref{lemma:boosting-failure-prob} to the next subsection. Once we have \Cref{lemma:boosting-failure-prob}, a union bound over all $K$ many $\calD_r$'s shows that the algorithm fails with probability less than $0.01$. When it does not fail, the following theorem guarantees generalization. 

\begin{theorem} \label{thm:generalize}
 Let $c : \calX \to \{1,-1\}$ be an unknown concept , $\calH$ be a hypothesis class of VC dimension $d$, and $\calD$ be an arbitrary distribution over $\calX$. Let $S$ be a training set of size $m$ randomly sampled from $\calD^m$. Let $\hat{h}$ be the voting classifier generated by \Cref{algo:round-query-boosting} ($g(x) = \frac{1}{k} \sum_{i=1}^k h_i(x)$ and $\hat{h}(x) = \mathrm{sign}(g(x))$). It must have large margins $c(x)g(x) \ge \gamma$ for all $x \in S$. It then follows that, with probability $1-\delta$,
 \begin{align*}
     \loss_{\calD}(f) \le \agen \cdot \frac{d \ln(m) \ln (m/d) + \ln(1/\delta)}{\gamma^2 m}  
 \end{align*}
where $\agen$ is a universal constant.
\end{theorem} 
\begin{proofsk}
We will prove in \Cref{appendix:adaboost} that by standard analysis of \textsc{AdaBoost}, $\hat{h}$ will satisfy the large margin property $c(x)g(x) \ge \gamma/16$ for all $x \in S$. The rest follows from Breiman’s min-margin bound \cite{breiman1999prediction}.
\end{proofsk}
\begin{proofof}{\Cref{theo:upper-bound-formal}}
    Conditioning on the algorithm does not fail, it follows from \Cref{thm:generalize} that $\loss_{\calD}(\hat{h}) \le 0.01$ with probability $0.99$ for $m = \ln(1/\gamma)\ln(d/\gamma) c' d / \gamma^2$ and a large constant $c'$. Since our algorithm only fails with probability $0.01$ from \Cref{lemma:boosting-failure-prob}, this proves \Cref{theo:upper-bound-formal}.
\end{proofof}

\subsection{Proof of \Cref{lemma:boosting-failure-prob}} \label{subsec:analysis}

To prove that the algorithm fails with small probability, we have to show that, at \Cref{Line:if}, such good $h \in \calH_k$ always exists. Note $\calH_k$ is generated by running weak learners on the multiset $T_{k,q}$'s. We will prove that at least one of the $T_{k,q}$'s is an ``$\eps$-apprxoimation'' of $\calD_{r}$, so that the hypothesis generated from running weak learner on $T_{k,q}$ will also have $O(\gamma)$ for $\calD_r$. 

\begin{definition}[$\eps$-approximation]
    A multiset $T$ is an $\eps$-approximation for a hypothesis class $\calH$ if for any $h \in \calH$, we have $|\loss_\calD(h) - \loss_{\calD_T}(h)| \le \eps$ where $\calD_T$ is the uniform distribution over $T$.
\end{definition}

Note that $T_{k,q}$ is drawn from $(\calD_{kR})^n$ not from $(\calD_r)^n$. For it to be a good approximation for $\calD_r$, we have to show that $\calD_{kR}$ and $\calD_r$ are close. Namely, the distribution does not change much after a few (less than $R$) updates of \textsc{AdaBoost}. Specifically, the only property we need is $(\eps, \delta)$-indistinguishability:

\begin{definition}[$(\eps, \delta)$-indistinguishable distributions]
    For any two distributions $\calD$ and $\calD'$ over $S$, they are $(\eps, \delta)$-indistinguishable if for any event $E \subseteq S$, it holds that $\Pr_{\calD} [E] \le \Pr_{\calD'} [E] \cdot e^{\eps} + \delta$ and $\Pr_{\calD_b} [E] \le \Pr_{\calD_a} [E] \cdot e^{\eps} + \delta$ where $\Pr_{\calD}[E]$ denotes $\Pr_{x \in \calD}[x \in E]$.
\end{definition}

To establish that $(\calD_{kR})^n$ not $(\calD_r)^n$ are $(\eps, \delta)$-indistinguishable, we start by upper bounding the max-divergence between $\calD_{kR}$ and $\calD_r$. 

\begin{lemma} \label{lemma:max-divergence}
The max-divergence between $\calD_{kR}$ and $\calD_r$, defined as $D_\infty(\calD, \calD') \defeq \ln(\sup_{x \in S} \calD(x) / \calD'(x))$, satisfies $D_\infty(\calD_{kR}, \calD_r)  \leq 2\gamma R$. When two distributions are switched, we also have $D_\infty(\calD_r, \calD_{kR}) \le 2\gamma R$.

Specifically, this implies that $\calD_{kR}$ and $\calD_r$ are $(2 \gamma R, 0)$-indistinguishable. 
\end{lemma}
\begin{proof}
    First of all, since $|kR - r| \leq R$ and max-divergence satisfies triangle inequality, it suffices to prove $D_\infty(\calD_{i}, \calD_{i+1}) \le 2 \gamma$ and $D_\infty(\calD_{i+1}, \calD_{i}) \le 2\gamma$ 
    for any $kR \leq i \leq r - 1$.

    Define exponential accumulated loss
    \begin{align*}
& Z_{i,j} \defeq \exp\left(-\sum_{r=0}^{i-1} c(x_j)h_r(x_j) \cdot w\right) \text{ for all } j \in [m], \\
& Z_i \defeq  \sum_{j=1}^{m} Z_{i,j}.
\end{align*}
    
    We can then express these two distributions as
    \begin{align*}
        & \calD_{i}(x_j) = Z_{i,j} / Z_i, \\
        & \calD_{i+1}(x_j)= Z_{i+1,j} / Z_{i+1}.
    \end{align*}
    Note that $Z_{i+1,j}$ will either be $Z_{i,j} e^{w}$ or $Z_{i,j} e^{-w}$, thus $\left|\ln \frac{Z_{i+1,j}}{Z_{i,j}}\right| \le w$. By definition
    \begin{align*}
        \frac{Z_{i+1}}{Z_{i}} = \frac{\sum_j Z_{i+1,j}}{\sum_j Z_{i,j}}, 
    \end{align*}
    hence we also have $ \left|\ln \frac{Z_{i+1}}{Z_{i}}\right| \le w$. Note $\gamma \in [0, 0.5]$ and 
    \begin{align}
        w = \ln ((1/2+\gamma/4)/(1/2-\gamma/4)) / 2 \le \ln(1+2\gamma)/2 \le \gamma, 
    \end{align} thus $\left|\ln \frac{\calD_{i}(x_j)}{\calD_{i+1}(x_j)}\right| \le \left|\ln \frac{Z_{i+1,j}}{Z_{i,j}}\right| + \left|\ln \frac{Z_{i+1}}{Z_{i}}\right| \le 2 \gamma$ for any $j \in [m]$. Therefore we have $D_\infty(\calD_i, \calD_{i+1}) \le 2 \gamma$ and $D_\infty(\calD_{i + 1}, \calD_i) \le 2 \gamma$.
\end{proof}

Then, loosely speaking, we will view $(\calD_{kR})^n$ and $(\calD_r)^n$ as $\calD_{kR}$ and $\calD_r$ ``composed'' $n$ times. This allows us to apply the advanced composition theorem from the differentiating privacy literature. 

\begin{lemma} [Advanced Composition \cite{dwork2010boosting}] \label{lem:advanced}
 For any integer $n>0$, any two $(\eps, \delta)$-indistinguishable distributions $\calD$ and $\calD'$, and any $\delta' > 0$, the product distributions $\calD^n$ and $(\calD')^n$ are $(\hat{\eps}, \hat{\delta})$-indistinguishable for 
 \begin{align*}
     & \hat{\eps} = n\eps (e^\eps-1) + \eps \sqrt{2n \ln (1/\delta')}, \\
     & \hat{\delta} = n \delta + \delta'.
 \end{align*}
\end{lemma}

\begin{corollary} \label{cor:dp}
$D_r$ and $D_{kR}$ are $(12 c' d R^{-2}, 1/4)$-indistinguishable.
\end{corollary}
\begin{proof}
    Using advanced composition (\Cref{lem:advanced}) and $n = c' d \gamma^{-2}$, setting $\delta' = 1/4$, we have 
    \begin{align*}
        \hat{\eps} \le 2n\eps^2 + \eps \sqrt{2n \ln(4)} \le 8 \gamma^2 R^2 n + 4 \gamma R \sqrt{n} = 8c'd R^2 + 4R \sqrt{c'd} \le 12c'd R^2
    \end{align*} 
    for large enough $c' > 1$. Here we used $2\gamma R \le 1$ and $e^x-1 \le 2x$ for $x \in [0,1]$.
\end{proof}

The rest of the proof proceeds as follows. First, By standard VC theory (\Cref{theo:VC-approximation}), for large enough $n$, a random subset $T$ drawn from $(\calD_r)^n$ is $\eps$-approximation of $\calD_r$ with good probability. Together with the fact that $(\calD_{kR})^n$ and $(\calD_r)^n$ are $(\eps, \delta)$ indistinguishable, we can argue that at least of the $T_{k,q}$'s will $\eps$-approximate $(\calD_r)^n$. This finishes the proof.

\begin{lemma}[\cite{vapnik2015uniform}] \label{theo:VC-approximation}
There exists a universal constant $c'>0$ for which the following is true. For any $0<\eps, \delta <1$, any distribution $\calD$ over $\calX$ and any hypothesis class $\calH$ of VC dimension $d$,  a random subset $T \sim \calD^{n}$ is an $\eps$-approximation for $\calH$ with probability $1-\delta$, provided that $n \ge 
 c'(d+\ln(1/\delta))\eps^{-2}$.
\end{lemma}

\begin{proofof}{\Cref{lemma:boosting-failure-prob}}

Let $c$ be a large enough constant. We denote by $G\subseteq \calX^{n}$ the set of multisets that $\gamma / 2$-approximates $\calD_r$. This is the set of ``good'' multisets. For a randomly sampled $T' \sim \calD_r^n$, it follows from \Cref{theo:VC-approximation} that $\Pr[T' \in G] \ge 3/4$. 

For a randomly sampled $T \sim \calD_{r}^n$, it follows from \Cref{cor:dp} that $\Pr[T'\in G] \le \Pr[T \in G] \cdot \exp(8c' d R^2) + 1/4$. Together with $\Pr[T' \in G] \ge 3/4$, it implies that $\Pr[T \in G] \ge \exp(-8 c' d R^2)/2$. 

Moreover, $T \in G$ means $\Loss_{\calD_r}(\calW(\calD_T)) \leq \Loss_{\calD_{T}}(\calW(\calD_T)) + \gamma / 2 \leq \frac{1}{2} - \gamma / 2$ where $\calW(\calD_T)$ is the hypothesis returned by the weak learner $\calW$ when it is called with distribution $\calD_T$. Hence, as long as one of $T_{k,q}$'s is in $G$, the algorithm will not fail for $\calD_r$ at \Cref{line:Q-samples}. When $Q \geq \exp(16 c'dR^2) \ln(1/\gamma)$, the failure probability for $\calD_r$ is
\begin{align*}
    \Pr[\texttt{fail for }\calD_r] & \le \Pr[\forall j\in [Q], T_{k,j} \notin G] \\
    & \le (1-\exp(-12 c' d R^2)/2)^Q \\
    & \le \exp(-\exp(-12 c' d R^2)Q/2) \\
    & \le \exp(-\exp(4c'dR^2)\ln(1/\gamma)/2) \\
    & \le 0.01 / (16 \gamma^{-2} \ln m) = 0.01/K,
\end{align*}
where we used the condition that $m =  \ln(1/\gamma)\ln(d/\gamma) c' d/\gamma^2$ in the last inequality. \end{proofof}

\subsection{Generalizing Lower Bound to Smooth Trade-Off} \label{subsec:lb}

In this subsection, we will generalize our lower bound to rounds fewer than $O(1/\gamma^2)$ by a slight modification.

\begin{theorem}\label{theo:trade-off-lower-bound}
There is a universal constant $\clb > 0$ for which the following is true. For every $R \ge 1, \gamma \in (0,1/2)$ and every $d,m\ge 1$, let $\calA$ be a boosting algorithm that uses $m$ samples and has parallel complexity $(p, t)$ where $p\le \min(\frac{\clb}{\gamma^2}, \exp(\clb\cdot d))/R$ and $t\le \exp(\clb R\cdot d)$. Then, there exists a domain $\calX$ of size $2m$, a hypothesis class $\calH\subseteq \{\pm 1\}^{|\calX|}$ of VC dimension $\frac{d}{\clb}$, a realizable distribution over $\calX\times \{\pm 1\}$, and a $\gamma$-weak learner $\weakl$ for $\calH$ such that
\begin{align} 
\Ex_{S\sim \calD^{m},\calA}[\loss_\calD(\strl^{\weakl}(S))] \ge \exp(-O(p \gamma^2R)).
\end{align}
\end{theorem}

To prove \Cref{theo:trade-off-lower-bound}, we use the construction in \Cref{sec:lower-bound-construct} but with $p\cdot R$ stages (namely, the hypothesis class will be $\calH = \bigcup_{j=1}^{pR} \calH^{(j)}\cup \{c\}$ where each $\calH^{(j)}$ contains a special hypothesis $a^{(j)}$ and many purely random hypotheses). 

Then, we have the following observation.

\begin{claim}\label{claim:independence}
    Consider the setup as in \Cref{theo:trade-off-lower-bound}. With probability $1-\exp(\Omega(d))$, all the queries from $\calA$ can be answered by hypotheses from $\calH^{(1)},\dots,\calH^{(p\cdot R)}$.
\end{claim}

\begin{proof}
Fix a realization of $c\sim \{ \pm 1\}^{2m}$ and $S\sim \calD^{m}$. We prove by induction on $i\in [p]$ that, the queries from the first $i$ rounds can be answered by hypotheses from $\calH^{(1)},\dots,\calH^{(iR)}$ with probability $1-i\cdot \exp(-\Omega(d))$. 

Assume the claim has been established for $i-1$. Under the event that queries of the first $(i-1)$ rounds were answered by hypotheses from $\calH^{(1)},\dots,\calH^{(iR)}$, we get that each query $\calD'$ in the $i$-th round is independent of $\calH^{(iR+1)},\dots,\calH^{((i+1)R)}$, and that $\calH^{(iR+1)},\dots, \calH^{((i+1)R)}$ are all mutually independent. By the argument in \Cref{sec:low-fail}, the probability that $\calD'$ can be answered by a hypothesis from $\calH^{(iR+j)}$ is $1-\exp(-\Omega(d))$ for every fixed $j\in [R]$. It follows that the probability that $\calD'$ cannot be answered by any hypothesis from $\calH^{(iR+1)},\dots, \calH^{((i+1)R}$ is at most $\exp(-R\cdot \Omega(d))$. We can then union-bound over at most $\exp(cRd)$ queries and finish the proof for the $i$-th round.
\end{proof}

\begin{proofof}{\Cref{theo:trade-off-lower-bound}}
Given \Cref{claim:independence}, with probability $1-\exp(-\Omega(d))$, the interaction between $\calA$ and the weak learner can be simulated given only $\calH^{(1)},\dots,\calH^{(pR)}$. Then, the same ``coin game'' argument gives the desired lower bound for the loss function.
\end{proofof}
\section*{Acknowledgements}

We thank the SODA reviewers for pointing out the important references and various helpful comments. We are grateful to Yuzhou Gu for helpful discussion.
\bibliography{main}
\bibliographystyle{plain}
\appendix 
\section{Generalization of \textsc{AdaBoost}} \label{appendix:adaboost}

In this appendix, we provide a proof for \Cref{thm:generalize} following the standard analysis of \textsc{AdaBoost}. First we prove that $\hat{h}$ satisfies the large margin property $c(x)g(x) \ge \gamma/16$ for all training samples $x \in S$. Then the generalization guarantee follows from Breiman's min-margin bound. 

First, let exponential loss for round $k$ (which is also defined in the proof of \Cref{lemma:max-divergence}) be
    \begin{align*}
& Z_{i,j} \defeq \exp\left(-\sum_{r=0}^{i-1} c(x_j)h_r(x_j) \cdot w\right) \text{ for all } j \in [m], \\
& Z_i \defeq  \sum_{j=1}^{m} Z_{i,j}.
\end{align*}

The main observation is that the exponential loss decreases exponentially. 
\begin{claim}\label{claim:exploss}
    $Z_{k+1} \le Z_{k} \cdot \sqrt{1-\gamma^2/4}$.
\end{claim}
\begin{proof}
    First observe that if we only normalize $\calD_k$ at last, we have $\calD_k(i) = Z_{k, i} / Z_k$. To update $Z_{k+1}$ from $Z_k$, we have $Z_{k+1, i} = Z_{k,i} \cdot \exp(-c(x_i)h_k(x_i)w)$, and thus
    \begin{align*}
      Z_{k+1} & = \sum_{c(x_i) = h_k(x_i)} Z_{k+1, i} + \sum_{c(x_i) \neq h_k(x_i)} Z_{k+1, i}
\\      & = Z_k \left(\sum_{i:c(x_i) = h_k(x_i)} \calD_{k}(i) e^{-w}  + \sum_{i:c(x_i) \neq h_k(x_i)} \calD_{k}(i)e^{w} \right)
\\     & = Z_k ((1-\loss_{\calD_k}(h_k)) e^{-w} + \loss_{\calD_k}(h_k)e^{w})
\\     & \le Z_k ((1/2+\gamma/4)e^{-w}+(1/2-\gamma/4)e^{w})
\\     & = Z_k \cdot 2 \sqrt{(1/2+\gamma/4)(1/2-\gamma/4)} = Z_k \sqrt{1 - \gamma^2/4}
    \end{align*}
    where $\loss_{\calD_k}(h_k) \le 1/2 - \gamma/4$ and setting it to $1/2 - \gamma/4$ maximizes the term. 
\end{proof}

Now, since the exponential loss is exponentially small, the voting classifier will have large margins on every sample.

\begin{claim} \label{claim:large-margin}
	Let $g$ be produced by \Cref{line:g-def} of \Cref{algo:round-query-boosting}. Setting $K = 16\ln m \cdot \gamma^{-2}$, for every training sample $(x_i, c(x_i)) \in (S,c(S))$, it holds that $c(x_i)g(x_i) \ge \gamma/16$.
\end{claim}

\begin{proof}
    We will prove by contradiction. Suppose exists some $x_{i^*}$ that $c(x_{i^*})g(x_{i^*}) < \gamma/16$, we have
    \begin{align*}
        Z_K \ge Z_{k, {i^*}} > \exp(-wK\gamma/16) \ge \exp(-K \gamma^2/16) = 1/m
    \end{align*}
    where we used the fact $\gamma \in [0, 0.5]$ and 
    \begin{align}
        w = \ln ((1/2+\gamma/4)/(1/2-\gamma/4)) / 2 \le \ln(1+2\gamma)/2 \le \gamma. \label{ineq:w}
    \end{align}
    On the other hand, from \Cref{claim:exploss}, we have
    \begin{align*}
        Z_K \le Z_0 \cdot (1-\gamma^2/4)^{K/2} \le m \cdot \exp(-K\gamma^2/8) = 1/m,
    \end{align*}
    which leads to a contradiction.
\end{proof}

Assuming the large margin property, the rest of \Cref{thm:generalize} follows from Breiman's min-margin bound \cite{breiman1999prediction}.

\end{document}